\documentclass[wcp]{jmlr}

\usepackage{longtable}

\usepackage{booktabs}

\usepackage{lineno}

\RenewDocumentEnvironment{proof}{o}{%
  \par\vspace{1mm}%
  \noindent{\bfseries\upshape \IfValueTF{#1}{#1}{Proof}\ }%
}{%
  \hspace*{\fill}\BlackBox%
  \par\vspace{2mm}%
}

\usepackage{graphicx}

\usepackage[dvipsnames]{xcolor}

\usepackage{pifont}
\newcommand{\cmark}{\ding{51}}
\newcommand{\xmark}{\ding{55}}

\newtheorem{assumption}[theorem]{Assumption}

\usepackage{listings}
\usepackage{xcolor}

\makeatletter
\let\Ginclude@graphics\@org@Ginclude@graphics 
\makeatother

\jmlrvolume{304,}
\jmlryear{2025}
\jmlrworkshop{ACML 2025}

\title[Information-Based Exploration via Random Features]{Information-Based Exploration via Random Features\\ for Reinforcement Learning}

\author{\Name{Waris Radji} \Email{waris.radji@inria.fr}\\
  \Name{Odalric-Ambrym Maillard} \Email{odalric.maillard@inria.fr}\\
\addr Univ. Lille, Inria, CNRS, Centrale Lille, UMR 9189-CRIStAL, F-59000 Lille, France}

\editors{Hung-yi Lee and Tongliang Liu}

\begin{document}

\makeatletter
\let \@jmlrpages \@empty
\makeatother

\maketitle

\begin{abstract}
Representation learning has enabled classical exploration strategies to be extended to deep Reinforcement Learning (RL), but often makes algorithms more complex and theoretical guarantees harder to establish. We introduce Random Feature Information Gain (\texttt{RFIG}), grounded in Bayesian kernel methods theory, which uses random Fourier features to approximate information gain and compute exploration bonuses in non-countable spaces. We provide error bounds on information gain approximation and avoid the black-box aspects of neural network-based uncertainty estimation, for optimism-based exploration. We present practical details that make \texttt{RFIG} scalable to deep RL scenarios, enabling smooth integration into standard deep RL algorithms. Experimental evaluation across diverse control and navigation tasks demonstrates that \texttt{RFIG} achieves competitive performance with well-established deep exploration methods while offering superior theoretical interpretation.
\end{abstract}

\section{Introduction}
%

In Reinforcement Learning (RL), agents learn optimal decision-making strategies through trial-and-error interactions with an environment, receiving rewards or penalties that guide their learning process \citep{sutton1998introduction}. A fundamental challenge is the exploration-exploitation tradeoff, where agents must balance between exploiting current knowledge to maximize immediate rewards and exploring new actions to potentially discover better long-term strategies. In this paper, we focus our attention on the exploration part in continuous and high-dimensional problems. In small-scale environments like Multi-Armed Bandits (MABs) and discrete Markov Decision Processes (MDPs), an effective strategy is the Optimism in Face of Uncertainty (OFU), which operates on the principle that when an agent lacks sufficient information about certain states, it should assume they may yield high rewards, thereby encouraging exploration of these uncertain regions \citep{auer2002finite,auer2008near}. This principle, theoretically grounded in MABs, was adapted to more general problems, through an \textit{exploration bonus}, where the reward obtained by the learner is augmented, typically $r_t + \beta r_t^+$, where at interaction step $t$, $r_t$ is the reward given by the MDP, $r_t^+$ is the exploration bonus and $\beta \geq 0$ is an parameter that control exploration strength. The literature often considers bonus in the form of $1 / \sqrt{n_t({s})}$, where $n_t({s})$ is the number of times the agent has visited state $s$ at interaction $t$: the more we visit a state, the more certain we are about the estimations \citep{strehl2008analysis}. However, in \textit{non-countable} spaces, which we consider here, this bonus is not straightforward to implement: the probability of visiting the same state twice can be zero, and direct count-based exploration becomes meaningless.

\paragraph{Research question.} This fundamental challenge has led to the development of deep learning-based exploration strategies, where neural networks (NNs) exploit \textit{representation learning} to learn a proxy of uncertainty or \textit{pseudo-counts}. 
Traditional deep representation learning approaches, while empirically successful \citep{bellemare2016unifying,pathak2017curiosity,badia2020never}, suffer from limited interpretability that hinders theoretical understanding and creates hyperparameter sensitivity. This brittleness arises from complex interactions between optimization dynamics and problem structure, where gradient-based algorithms and backpropagation show fragile dependencies on learning rate schedules that can destabilize training and require parameter adjustments spanning orders of magnitude across domains \citep{glorot2010understanding}. Additionally, architectural choices and regularization strategies depend on environmental characteristics, creating optimization landscapes where effective hyperparameter settings often transfer poorly and fail dramatically in new contexts. This sensitivity means that even algorithmically sound approaches can exhibit dramatic performance degradation when operating under suboptimal hyperparameter regimes \citep{henderson2018deep}.
These observations raise an interesting question:\vspace{-0.5em}
\begin{center}
    \textit{
    How can we design alternative exploration mechanisms for deep RL that offer theoretical interpretability and computational tractability while achieving competitive empirical performance relative to standard deep methods?
} 
\end{center}\vspace{-0.5em}
\paragraph{Interesting directions.} A promising direction lies in kernel methods, which provide theoretically grounded uncertainty quantification through closed-form solutions rather than iterative gradient-based optimization. These methods avoid the extensive hyperparameter tuning and training instabilities inherent to neural architectures \citep{srinivas2009gaussian}, while maintaining the ability to capture complex nonlinear patterns. From a \textit{Bayesian} perspective, the concept of \textit{information gain}, provides a grounded approach to create exploration bonus that go beyond discrete spaces: poorly visited states are very uncertain and could lead to high information gain, making them attractive targets for exploration while naturally diminishing the bonus as states become well-explored and their uncertainty decreases \citep{kolter2009near}. However, \textit{vanilla} kernel methods suffer from cubic computational complexity, limiting their scalability to the large sample sizes required in deep RL.
\paragraph{Contributions.} In this paper, to tackle the problem of OFU exploration in uncountable spaces, we introduce a novel exploration bonus for deep RL: Random Feature Information Gain (\texttt{RFIG}). Our bonus is directly derived from information gain quantification in Bayesian kernel methods, which we combine with recent advances that allow these methods to become scalable. An important component of our approach is to exploit random features \citep{rahimi2007random}, which enable capturing complex nonlinear spatial patterns in high-dimensional data and approximate kernels, to handle the cubic scaling in the number of samples. Unlike pure deep learning approaches, \texttt{RFIG} computes exploration bonuses via closed-form solutions, eliminating backpropagation complexity and hyperparameter brittleness while maintaining theoretical interpretability.
\paragraph{Outline.} We first derive \texttt{RFIG} from Bayesian kernel methods and random features (Section~\ref{section:derivation}). We then provide approximation error bounds (Section~\ref{section:error}) and apply them to random Fourier features (Section~\ref{section:rff}). Finally, we detail algorithmic integration with deep RL (Section~\ref{section:integration}) and demonstrate effectiveness across diverse tasks (Section~\ref{section:experiments}).

\section{Related Work}
Exploration remains a fundamental challenge in RL, particularly in environments with sparse rewards or large state spaces. We review existing approaches, progressing from general methods through kernelized MDPs to deep RL exploration strategies.

\paragraph{Exploration foundations.} The exploration-exploitation tradeoff was first formalized in MABs through OFU-based algorithms \citep{auer2002finite} and Thompson Sampling \citep{thompson1933likelihood,chapelle2011empirical}. Recent approaches like Information Directed Sampling \citep{russo2014learning} and Minimum Empirical Divergence \citep{honda2010asymptotically} directly formalize information gain in their objectives. These principles have been extended to tabular MDPs \citep{auer2008near,osband2013more,pesquerel2022imed} and subsequently to continuous spaces through kernelization and deep RL.

\paragraph{Kernelized MDPs.} Kernel methods extend bandit exploration principles to MDPs with theoretically grounded uncertainty quantification. In bandits, \texttt{GP-UCB} achieves provable regret guarantees using Gaussian Process posteriors \citep{srinivas2009gaussian,valko2013finite}. This framework extends to MDPs where kernels encode similarity structure over state-action spaces \citep{morere2018bayesian,chowdhury2019online,domingues2021kernel}, and demonstrates that kernel-based uncertainty quantification enables theorically-grounded exploration in continuous MDPs but remains limited to relatively small-scale problems due to computational constraints.

\paragraph{Deep RL exploration.} Scaling to high-dimensional spaces typically requires representation learning. Curiosity-driven methods \citep{pathak2017curiosity} and episodic novelty approaches \citep{badia2020never} learn embeddings for exploration. Information-theoretic methods maximize information gain through NN ensembles \citep{houthooft2016vime,nikolov2018information,sukhija2024maxinforl} or prediction disagreement \citep{osband2016deep,azizzadenesheli2018efficient}. While effective, these approaches couple exploration with representation learning, creating hyperparameter sensitivity \citep{glorot2010understanding,henderson2018deep}.
A smaller line of work separates exploration from representation learning. \texttt{RND} \citep{burda2018exploration} uses random features by training a network to predict a fixed random network's outputs, where prediction error signals novelty. This demonstrates feature learning is not required for effective exploration, though \texttt{RND} lacks theoretical grounding. \cite{yang2024exploration} addresses this through connections to pseudo-counts.

\paragraph{Bridging kernels and deep RL.} Recent attempts to apply kernel methods in deep RL reveal different approaches. \cite{ma2024highly} uses random Fourier features with kernel density estimation but requires user-defined success criteria for specific environments. \cite{blau2019bayesian} develops a "Bayesian curiosity module" using posterior variance from learned kernels but suffers from cubic complexity, limiting scalability. They mention RFFs as future work, which we implement in this paper.

\section{Background on Information Gain, RL and Scalable Kernels} \label{section:background}

This section establishes the theoretical foundations: information gain for exploration, Bayesian kernel methods for uncertainty quantification, and random Fourier features for scalability.

\paragraph{Exploration in RL via information gain.} An agent interacts with a discounted MDP $\mathbf{M} = (\mathcal{S}, \mathcal{A}, \mathbf{r}, \mathbf{p}, \gamma)$ to learn a policy $\pi: \mathcal{S} \rightarrow \Pr(\mathcal{A})$ maximizing expected cumulative reward $J(\pi) = \mathbb{E}_{\pi, \mathbf{p}}\left[\sum_{t=0}^{\infty} \gamma^t \mathbf{r}(s_t, a_t)\right]$ \citep{sutton1998introduction}. A standard approach augments rewards with exploration bonuses \citep{strehl2008analysis}: $\mathbf{r}_{\text{total}}(s, a) = \mathbf{r}(s, a) + \beta \mathbf{r}^{+}(s, a)$, where $\beta > 0$ controls exploration strength. The widely-used bonus $1/\sqrt{n(s)}$ (with $n(s)$ the visit count for state $s$) implicitly maximizes information gain \citep{bellemare2016unifying}.
To formalize this, consider learning an unknown function $f: \mathcal{X} \rightarrow \mathbb{R}$ from noisy data $\mathcal{D}_n = \{(x_i, y_i)\}_{i=1}^n$ where $y_i = f(x_i) + \eta_i$. The Bayesian posterior $p(f \mid \mathcal{D}_n)$ encodes uncertainty about $f$. The expected information gain from querying $x_*$ is
\begin{equation} \label{eq:expected-information-gain}
\operatorname{IG}(x_* \mid \mathcal{D}_n) = H(f \mid \mathcal{D}_n) - \mathbb{E}_{Y_*}[H(f \mid \mathcal{D}_n \cup \{(x_*, Y_*)\})]
\end{equation}
where $H(f \mid \mathcal{D}) = -\int p(f \mid \mathcal{D}) \log p(f \mid \mathcal{D}) \, df$ is the differential entropy \citep{cover1999elements}. This criterion, connect with active inference \citep{settles2009active,friston2015active}.

\paragraph{Bayesian kernel methods.} As a alternative to neural exploration, we employ kernel methods that provide principled uncertainty quantification through implicit mapping to reproducing kernel Hilbert spaces $\mathcal{H}_k$ \citep{aronszajn1950theory,scholkopf2001generalized}. A positive semi-definite kernel $k: \mathcal{X} \times \mathcal{X} \rightarrow \mathbb{R}$ enables high-dimensional computations using only pairwise similarities. In Bayesian kernel ridge regression \citep{saunders1998ridge,jaakkola1999probabilistic} with regularization $\lambda > 0$, the posterior mean and variance are
\begin{align}
\mu_n(x) = \mathbf{k}_n(x)^T (\mathbf{K}_n + \lambda \mathbf{I}_n)^{-1} \mathbf{y}_n \qquad 
\sigma_n^2(x) = k(x, x) - \mathbf{k}_n(x)^T (\mathbf{K}_n + \lambda \mathbf{I}_n)^{-1} \mathbf{k}_n(x)
\end{align}
where $\mathbf{K}_n \in \mathbb{R}^{n \times n}$ with $[\mathbf{K}_n]_{ij} = k(x_i, x_j)$ and $\mathbf{k}_n(x) = [k(x_1, x), \ldots, k(x_n, x)]^T$. This is equivalent to a Gaussian process $f \sim \mathcal{GP}(0, k(x, x'))$ with noise $\sigma^2 = \lambda$ \citep{williams1995gaussian}. However, inverting $(\mathbf{K}_n + \lambda \mathbf{I}_n)$ requires $\mathcal{O}(n^3)$ operations, prohibitive for large datasets.
 Random Fourier Features (RFFs) \citep{rahimi2007random} resolve this bottleneck by approximating kernels with explicit finite-dimensional mappings. For shift-invariant kernels $k(x, x') = k(x - x')$, Bochner's theorem \citep{bochner1959lectures} enables $k(x, x') \approx \boldsymbol{\phi}(x)^T \boldsymbol{\phi}(x')$ where
\begin{equation}
\boldsymbol{\phi}(x) = \sqrt{\frac{2}{D}} \begin{bmatrix} \cos(\boldsymbol{\omega}_1^T x + b_1) \\ \vdots \\ \cos(\boldsymbol{\omega}_D^T x + b_D) \end{bmatrix}
\end{equation}
with $\boldsymbol{\omega}_i \sim p(\boldsymbol{\omega})$ from the kernel's spectral density (Fourier transform of $k$) and $b_i \sim \text{Uniform}[0, 2\pi]$. For the RBF kernel $k(x, x') = \exp(-\|x - x'\|^2 / 2\ell^2)$, the length-scale $\ell$ determines $p(\boldsymbol{\omega}) = \mathcal{N}(0, \ell^{-2}\mathbf{I})$. Applying the Woodbury identity \citep{woodbury1950inverting} with feature matrix $\boldsymbol{\Phi}_n \in \mathbb{R}^{n \times D}$ yields
\begin{align}
\mu_n(x) &= \boldsymbol{\phi}(x)^T (\boldsymbol{\Phi}_n^T \boldsymbol{\Phi}_n + \lambda \mathbf{I}_D)^{-1} \boldsymbol{\Phi}_n^T \mathbf{y}_n \\
\sigma_n^2(x) &= \boldsymbol{\phi}(x)^T \boldsymbol{\phi}(x) - \boldsymbol{\phi}(x)^T (\boldsymbol{\Phi}_n^T \boldsymbol{\Phi}_n + \lambda \mathbf{I}_D)^{-1} \boldsymbol{\phi}(x).
\end{align}
This reduces computational complexity from $\mathcal{O}(n^3)$ to $\mathcal{O}(D^3)$, enabling efficient uncertainty quantification that scales with feature dimension $D$ rather than dataset size $n$.

\section{Random Feature Information Gain}

Before looking at how information gain is implemented in a RL training loop to promote exploration, we now derive our Random Feature Information Gain (\texttt{RFIG}).
The derivation proceeds in three steps: (1) express GP information gain in terms of posterior variance, (2) approximate the kernel matrix using random features, (3) apply matrix identities to obtain the final $\mathcal{O}(D^3)$ form. All detailed proofs of this section can be found in Appendix~\ref{appendix:proofs}.

\subsection{Derivation} \label{section:derivation}

We start by recalling the information gain in the Gaussian process framework using the entropy reduction formulation, as described in \eqref{eq:expected-information-gain}. Consider a Gaussian process, that we defined in Section~\ref{section:background}, $f \sim \mathcal{GP}(0, k(\cdot, \cdot))$ with observation noise $\eta \sim \mathcal{N}(0, \sigma^2)$. Given current data $\mathcal{D}_n = \{(x_i, y_i)\}_{i=1}^n$, the posterior entropy can be expressed using the kernel matrix $\mathbf{K}_n$ with
$
H(f \mid \mathcal{D}_n) = \frac{1}{2} \log \det(2\pi e (\mathbf{K}_n + \sigma^2 \mathbf{I}_n)^{-1}).
$
When we add a new observation $(x_*, y_*)$ to our dataset, obtaining $\mathcal{D}_{n+1} = \mathcal{D}_n \cup \{(x_*, y_*)\}$, the posterior distribution changes.

\begin{definition}[Information gain in GP \citep{lawrence2002fast}]\footnote{This formulation is equivalent to what they term the ``differential entropy score".} \label{eq:ig-gp}
The information gain, as defined in \eqref{eq:expected-information-gain}, can be expressed for a query point $x_*$ in GP, as
\begin{align*} \label{eq:ig-gp}
\operatorname{IG}(x_* \mid \mathcal{D}_n) &= H(f \mid \mathcal{D}_n) - \mathbb{E}_{Y_*}[H(f \mid \mathcal{D}_{n+1})] \\
&= \frac{1}{2} \log \det(2\pi e (\mathbf{K}_n + \sigma^2 \mathbf{I}_n)^{-1}) - \mathbb{E}_{Y_*}\left[\frac{1}{2} \log \det(2\pi e (\mathbf{K}_{n+1} + \sigma^2 \mathbf{I}_{n+1})^{-1})\right] \\
&= \frac{1}{2} \log \det(\mathbf{K}_{n+1} + \sigma^2 \mathbf{I}_{n+1}) - \frac{1}{2} \log \det(\mathbf{K}_n + \sigma^2 \mathbf{I}_n) = \boxed{\frac{1}{2} \log\left(1 + \frac{\sigma_n^2(x_*)}{\sigma^2}\right)}.
\end{align*}
\end{definition}

\paragraph{Challenges.} However, computing $\sigma_n^2(x_*)$ requires inverting the $n \times n$ matrix $(\mathbf{K}_n + \sigma^2 \mathbf{I}_n)$, which scales as $\mathcal{O}(n^3)$ and becomes prohibitive for huge datasets. To address this computational bottleneck, we next develop a random feature approximation that reduces complexity from $\mathcal{O}(n^3)$ to $\mathcal{O}(D^3)$, where $D$ is the number of features (Proposition~\ref{prop:rf-information-gain}). Additionally, practical implementation requires careful selection of the kernel length-scale $\ell$, which controls the smoothness assumptions and generalization behavior, and the number of random features $D$, which determines the approximation quality versus computational cost tradeoff. Later in the paper, we will propose insights for choosing $D$ based on theoretical error bounds (Corollary~\ref{cor:feature-requirement}) and adaptive length-scale selection recommandations (Section~\ref{section:integration}).

\begin{proposition}[Information Gain via Random Features]
\label{prop:rf-information-gain}
Consider a random feature transformation $\boldsymbol{\phi}: \mathcal{X} \rightarrow \mathbb{R}^D$ that approximates a shift-invariant kernel $k(x, x') \approx \boldsymbol{\phi}(x)^T \boldsymbol{\phi}(x')$ \citep{rahimi2007random}. The information gain (Definition~\ref{eq:ig-gp}) can be approximated as
\begin{equation}
\hat{\operatorname{IG}}(x_* \mid \mathcal{D}_n) = \frac{1}{2} \log \left( 1 + \boldsymbol{\phi}(x_*)^T(\boldsymbol{\Phi}_n^T \boldsymbol{\Phi}_n + \lambda\mathbf{I}_D)^{-1} \boldsymbol{\phi}(x_*) \right)
\end{equation}
where $\boldsymbol{\Phi}_n \in \mathbb{R}^{n \times D}$ is the feature matrix with rows $\boldsymbol{\phi}(x_i)^T$ for $i = 1, \ldots, n$, and $\lambda > 0$ is the regularization parameter.
\end{proposition}

\subsection{Error Bounds} \label{section:error}

To provide theoretical guarantees for our approach, we establish error bounds for RFIG under uniform kernel convergence assumptions. Our analysis serves two key purposes: (1) quantifying how errors in kernel approximation propagate to information gain estimates, and (2) determining the number of random features $D$ required to achieve a desired approximation accuracy $\varepsilon$ with high probability.
We proceed by first bounding the error in posterior variance estimation, then using this result to establish guarantees for information gain approximation, and finally applying our general framework to RFFs.
 Our analysis relies on three standard assumptions commonly employed in the random features literature \citep{rahimi2007random, sutherland2015error}.
\begin{assumption}[Uniform kernel approximation] \label{ass:uniform}
The random feature map $\boldsymbol{\phi}(x): \mathcal{X} \rightarrow \mathbb{R}^D$ provides a uniform approximation to the kernel $k(x, x')$ over the domain:
\begin{equation}
    \mathbb{P}\Big [\sup_{x, x' \in \mathcal{X}} |\boldsymbol{\phi}(x)^\top \boldsymbol{\phi}(x') - k(x, x')| \geq \epsilon\Big ] \leq \delta(\epsilon; d, D).
\end{equation}
\end{assumption}
\begin{assumption}[Regularization scaling] \label{ass:regularization}
The regularization parameter scales linearly with sample size: $\lambda = n \lambda_0$ for some $\lambda_0 > 0$.
\end{assumption}
\begin{assumption}[Bounded kernel] \label{ass:bounded}
The kernel is bounded: $|k(x, x')| \leq \kappa$ for all $x, x' \in \mathcal{X}$.
\end{assumption}
Assumption~\ref{ass:uniform} is the core requirement for random feature methods and holds for RFFs under mild conditions on the input domain \citep{rahimi2007random}. Assumption~\ref{ass:regularization} ensures that the regularization term remain properly balanced as sample size grows, preventing regularization from either dominating or vanishing asymptotically, which is useful for deriving clean convergence rates and consistency results. Assumption~\ref{ass:bounded} is satisfied by most practical kernels including RBF and Matérn kernels.

\paragraph{Posterior variance error.} Since information gain is fundamentally determined by posterior variance \eqref{eq:expected-information-gain}, we first establish how kernel errors propagate to variance estimates.

\begin{proposition}[Posterior variance error bound] \label{prop:variance-error}
Under Assumptions~\ref{ass:uniform},~\ref{ass:regularization}, and~\ref{ass:bounded}, the error in posterior variance estimation when using random features is bounded by:
\begin{equation}
    \forall x \in \mathcal{X}, |\hat{\sigma}^2_n(x) - \sigma^2_n(x)| \leq \epsilon \left(1 + \frac{\kappa^2}{\lambda_0^2} + \frac{2\kappa}{\lambda_0} + \frac{\epsilon}{\lambda_0}\right),
\end{equation}
where $\epsilon = \sup_{x, x' \in \mathcal{X}} |\boldsymbol{\phi}(x)^\top \boldsymbol{\phi}(x') - k(x, x')|$.
\end{proposition}
This result shows that variance estimation error scales linearly with kernel approximation quality $\epsilon$ and exhibits the expected dependence on regularization strength.

\paragraph{Information gain error.} The connection between posterior variance and information gain enables us to translate variance errors into information gain guarantees (Lemma~\ref{lem:geometric_mean_bound}). 

\begin{proposition}[RFIG error bound] \label{prop:ig-error}
Under Assumptions~\ref{ass:uniform},~\ref{ass:regularization}, and~\ref{ass:bounded}, the error in RFIG approximation is bounded by:
\begin{equation}
    |\operatorname{IG}(x| \mathcal{D}_n) - \hat{\operatorname{IG}}(x| \mathcal{D}_n)| \leq \frac{\epsilon(\lambda_0+\kappa)^2+\epsilon^2\lambda_0}{2n\lambda_0^3},
\end{equation}
where $\epsilon = \sup_{x, x' \in \mathcal{X}} |\boldsymbol{\phi}(x)^\top \boldsymbol{\phi}(x') - k(x, x')|$.
\end{proposition}
Our bound exhibits some properties: the error decreases with sample size $n$ (consistency), scales with kernel approximation quality $\epsilon$ (approximation dependence), and reveals a regularization trade-off where stronger $\lambda_0$ tightens the bound but may over-smooth posteriors.

\subsection{Application to Random Fourier Features} \label{section:rff}

We apply our general bound to RFFs by using existing uniform convergence results.

\begin{proposition}[RFF uniform convergence \cite{rahimi2007random}] \label{prop:rff-convergence}
Let $\mathcal{X} \subset \mathbb{R}^d$ be compact with diameter $\operatorname{diam}(\mathcal{X})$ and $k$ a shift-invariant kernel with unit maximum and Fourier transform $P(\omega)$. Let $\sigma^2_p = \mathbb{E}_P[\|\omega\|^2]$. For RFF mapping $\boldsymbol{\phi}$ and any $\epsilon > 0$:
\begin{equation}
\Pr\big[\|\boldsymbol{\phi}^\top \boldsymbol{\phi} - k \|_\infty \geq \epsilon \big ] \leq c \left ( \frac{\sigma_p \operatorname{diam}(\mathcal{X})}{\epsilon}\right)^2 \exp \left ( - \frac{D\epsilon^2}{8(d + 2)} \right ),
\end{equation}
$c=256$ in \cite{rahimi2007random}, tightened to 66 in \cite{sutherland2015error}.
\end{proposition}


\begin{corollary}[Feature dimension requirement] \label{cor:feature-requirement}
To achieve approximation error \\ $\sup_x|\operatorname{IG}(x| \mathcal{D}_n) - \hat{\operatorname{IG}}(x| \mathcal{D}_n)| \leq \varepsilon$ with probability at least $1-\delta$, it suffices to choose
\begin{equation}
    D = \mathcal{O}\left(\frac{d}{\epsilon_k^2} \log \frac{\sigma_p \operatorname{diam}(\mathcal{X})}{\epsilon_k \delta}\right),
\end{equation}
where $\epsilon_k = \dfrac{2n\lambda_0^3\varepsilon}{(\lambda_0+\kappa)^2}$ when $\varepsilon$ is sufficiently small.
\end{corollary}
Even if $\operatorname{diam}(\mathcal{X})$ is generally not know in a RL context, this result provides practical guidance for hyperparameter selection: the required feature dimension $D$ scales linearly with problem dimension $d$ and logarithmically with desired accuracy. Importantly, $D$ decreases with sample size $n$ through $\epsilon_k$, reflecting that larger datasets permit coarser kernel approximations while maintaining the same information gain accuracy. This theoretical foundation justifies our approach and enables confident deployment in practical exploration scenarios.

\section{\texttt{RFIG} for Efficient Exploration in RL}

This paper aims to apply \texttt{RFIG} for improving optimism-based exploration in deep RL. This section outlines the key algorithmic components and implementation considerations that enable efficient and scalable integration with existing deep RL agents.

\subsection{Details that Matter} \label{section:integration}
While Algorithm~\ref{algo} outlines the core \texttt{RFIG} integration with deep RL, successful implementation requires attention to several practical considerations. However, these hyperparameter choices are minimal compared to neural network approaches, which typically demand extensive tuning of learning rates, network architectures, regularization schemes, and optimization schedules. This subsection presents the key considerations that determine \texttt{RFIG}'s effectiveness in practice, demonstrating the relative simplicity of our kernel-based approach.

\begin{algorithm2e}[ht]
\caption{\texttt{RFIG} for exploration}
\label{algo}
\KwIn{RFF Feature map $\boldsymbol{\phi}_\ell: \mathcal{X} \rightarrow \mathbb{R}^D$ with \textcolor{ForestGreen}{length-scale $\ell \propto \sqrt{\bar d}$}, regularization $\lambda > 0$, \textcolor{Fuchsia}{subsample ratio $\rho \in (0,1]$}, environment $\mathbf{M}$, policy $\pi$, exploration scale $\beta > 0$.}

Initialize RFIG matrices $\boldsymbol{\Sigma}_0 \gets \lambda\mathbf{I}_D$ and $\boldsymbol{\Lambda}_0 \gets \lambda^{-1} \mathbf{I}_D$

\textcolor{BlueViolet}{Initialize state normalization parameters $(\mu_s, \sigma^2_s)$}

\For{$t \gets 1, 2, \cdots$}{
    Collect $N$ transitions $\mathcal{D} = \{(s_i, a_i, r_i, s'_i)\}_{i=1}^N$ with policy $\pi$ in environment $\mathbf{M}$
    
    \textcolor{BlueViolet}{Update normalization parameters with $\{s_i\}_{i=1}^N$, obtain normalized states $\{\bar{s}_i\}_{i=1}^N$}
    
    Compute information gain bonuses $\mathcal{R}^+ = \left\{r_i^+ = \frac{1}{2}\log\left(1 + \boldsymbol{\phi}_\ell(\bar{s}_i)^\top\boldsymbol{\Lambda}_{t-1} \boldsymbol{\phi}_\ell(\bar{s}_i)\right)\right\}_{i=1}^N$
    
    \textcolor{Fuchsia}{Subsample $\lfloor N\rho \rfloor$ states uniformly from $\{\bar{s}_i\}_{i=1}^N$ to form $\boldsymbol{\Phi}_t$ with rows $\boldsymbol{\phi}_\ell(\bar{s}_j)^\top$}
    
    Update $\boldsymbol{\Sigma}_t \gets \boldsymbol{\Sigma}_{t-1} + \boldsymbol{\Phi}_t^\top \boldsymbol{\Phi}_t$, then \textcolor{BrickRed}{$\boldsymbol{\Lambda}_t \gets \boldsymbol{\Sigma}_t^{-1}$}
    
    Update policy $\pi$ using any RL algorithm with $\mathcal{D}$ and bonuses $\mathcal{R}^+$
}
\end{algorithm2e}
\paragraph{\textcolor{ForestGreen}{length-scale selection.}} The length-scale $\ell$ controls the smoothness of the uncertainty estimates and should account for the curse of dimensionality. In high-dimensional spaces, typical distances between points scale as $\sqrt{\bar{d}}$ where $\bar{d}$ is the \textit{effective} input dimension \citep{hvarfner2024vanilla,xu2024standard}. Therefore, we recommend initializing $\ell \propto \sqrt{\bar d}$. To estimate the effective dimension from samples, we refer to Section 4 of \cite{valko2013finite}.
\paragraph{\textcolor{BlueViolet}{State normalization}\protect\footnote{\label{note}These details have shown beneficial for many deep exploration strategies in \cite{yuan2024rlexplore}.}.} We maintain running statistics $\mu_s$ and $\sigma_s^2$ to normalize states as $\bar{s} = (s - \mu_s)/\sigma_s$. This prevents scale differences across dimensions from dominating kernel computations and is critical for RFF effectiveness.
\paragraph{\textcolor{Fuchsia}{Subsampling Strategy}\protect\footref{note}.} The subsample ratio $\rho$ serves multiple purposes. The primary goal is to prevent information gain from shrinking too rapidly to zero as the number of samples grows, which would lead to premature exploration termination. Additionally, subsampling helps Newton-Schulz iterations converge faster since the covariance matrix $\boldsymbol{\Sigma}_t$ changes more slowly between updates, making warm starts more effective. This approach mirrors techniques in sparse Gaussian processes, where a subset of inducing points can represent the uncertainty structure of the entire dataset.
\paragraph{\textcolor{BrickRed}{Newton-Schulz matrix inversion.}} A key computational challenge in RFIG is efficiently maintaining the matrix $(\boldsymbol{\Phi}_n^T\boldsymbol{\Phi}_n + \lambda\mathbf{I}_D)^{-1}$ as new observations arrive. It's possible to employ the Newton-Schulz iteration \citep{schulz1933iterative}, which iteratively computes matrix inverses using 
    $\mathbf{X}_{k+1} = \mathbf{X}_k(2\mathbf{I} - \mathbf{A}\mathbf{X}_k)$.
This method converges quadratically to $\mathbf{A}^{-1}$ when $\|\mathbf{I} - \mathbf{A}\mathbf{X}_0\|_2 < 1$ and crucially allows using the previous iteration's result as a warm start for $\mathbf{X}_0$. Compared to Sherman-Morrison or Woodbury updates, more commonly considered, Newton-Schulz offers superior numerical stability by avoiding explicit small-number divisions and provides computational savings. Combined with its parallel structure that maps naturally to GPU architectures, Newton-Schulz is ideally suited for the frequent matrix updates required in online deep RL applications. Further details are in Appendix~\ref{appendix:newton}.

\subsection{Numerical Experiments} \label{section:experiments}

We evaluate \texttt{RFIG} by integrating it with Proximal Policy Optimization (\texttt{PPO}) algorithm \citep{schulman2017proximal}, following the non-episodic exploration framework described in \cite{burda2018exploration} for Random Network Distillation (\texttt{RND}). This allows for direct comparison with proven implementation practices for intrinsic motivation in deep RL. Following the \texttt{PPO+RND} architecture, we augment the standard \texttt{PPO} objective with \texttt{RFIG}-based intrinsic rewards. We maintain separate value networks for extrinsic and intrinsic rewards and normalize both rewards, as this has benefited many bonuses \citep{yuan2024rlexplore}. 

\paragraph{Baselines.} In addition to \texttt{RND}, we consider two other deep RL baselines that also follow the optimism under the face of uncertainty (OFU) principle. We include \texttt{\#Explo} from \cite{tang2017exploration}, which learns a hash function through an autoencoder architecture, maintains a hash table with visit counts, and employs standard count-based exploration bonuses. We also consider \texttt{VIME} \citep{houthooft2016vime}, which shares a similar spirit to our approach by targeting information gain for exploration. \texttt{VIME} learns dynamics with a Bayesian NN, a NN with probability distributions over weights rather than fixed parameters, and approximates information gain as the Kullback-Leibler divergence between the prior and posterior weight distributions of the network. Both methods require effective representation learning to function properly and reflect well the popular neural network-based exploration paradigm that works independently from the policy learning process, via reward bonuses.

\paragraph{Setup.} We adopt global hyperparameter settings proven effective for \texttt{PPO} across all experiments (detailed in Appendix~\ref{appendix:ppo_hyperparams}). The exploration coefficient $\beta$ is set to $0.5$ for all methods. Since we normalize exploration bonuses before integration, this coefficient does not affect the relative comparison between baselines.
For fair comparison across exploration baselines, we avoid extensive hyperparameter search and instead use common, well-established parameter values found in reference implementations. 
All baseline methods employ neural networks with $256 \times 256$ hidden layers and perform one gradient step per batch update using the Adam optimizer.
We initialize observation normalization with random trajectories for all methods and additionally estimate the effective dimension $\bar{d}$ for \texttt{RFIG}.
For \texttt{RFIG}-specific parameters, we use $D = 1024$ random features, regularization $\lambda = 1,$ subsample ratio $\rho = 6.25\%$, and length-scale $\ell = \sqrt{\bar d}$. 
We evaluate \texttt{RFIG} across four domains designed to test exploration capabilities. Classic control tasks (Acrobot, MountainCar) provide baseline comparisons in low-dimensional settings \citep{gymnax2022github}. For challenging continuous control, we use sparse reward variants of Brax locomotion tasks \citep{brax2021github}, where agents receive milestone rewards only upon reaching specific distance thresholds (Appendix~\ref{appendix:milestone_wrapper}). We include PointMaze navigation environments \citep{gymnasium_robotics2023github,pointax2025} and MinAtar tasks \citep{young19minatar,gymnax2022github}, miniature Atari implementations that demonstrate exploration needs extend beyond sparse reward settings. All experiments use 32 random seeds and 32 parallel environments with 128-step unrolls.

\begin{figure}[ht]
    \centering
    
    \begin{tabular}{@{}cccc@{}}
        \includegraphics[width=0.23\textwidth]{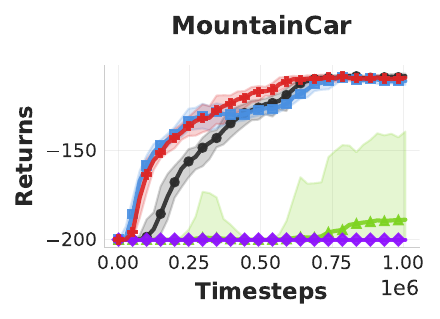} &
        \includegraphics[width=0.23\textwidth]{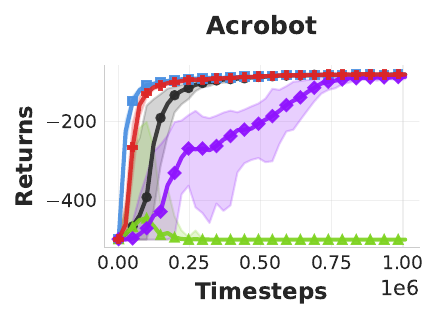} &
        \includegraphics[width=0.23\textwidth]{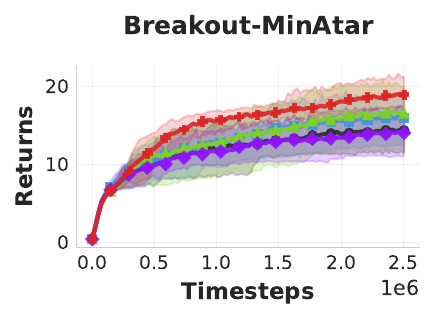} &
        \includegraphics[width=0.23\textwidth]{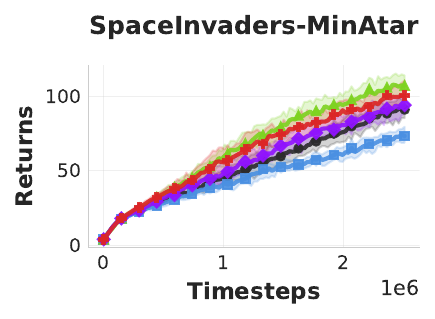} 
        \\[1em]

        \includegraphics[width=0.23\textwidth]{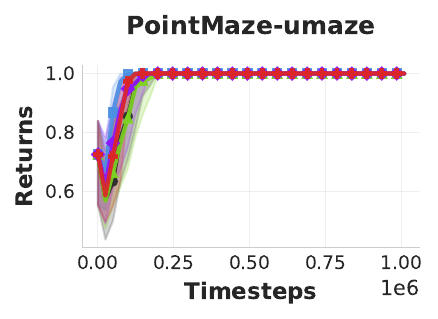} &
        \includegraphics[width=0.23\textwidth]{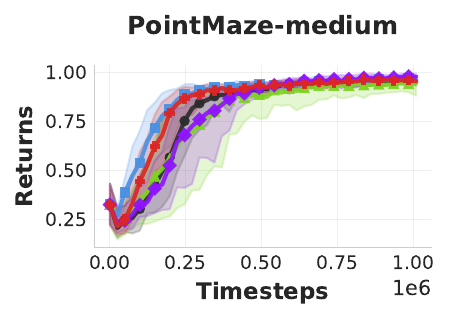} &
        \includegraphics[width=0.23\textwidth]{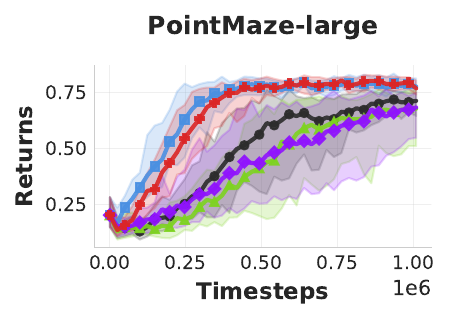} &
        \includegraphics[width=0.23\textwidth]{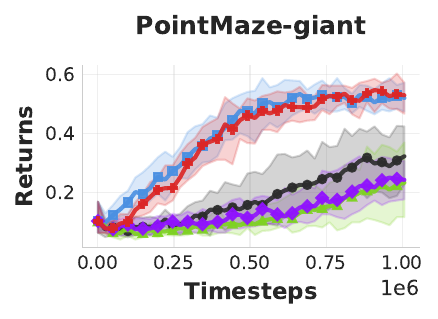} 
        \\[1em]

        \includegraphics[width=0.23\textwidth]{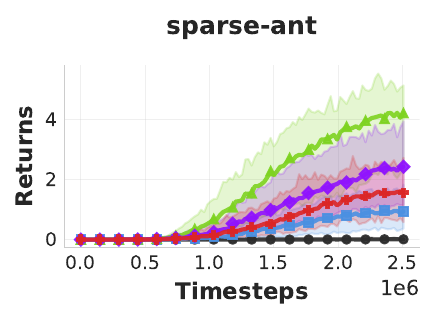} &
        \includegraphics[width=0.23\textwidth]{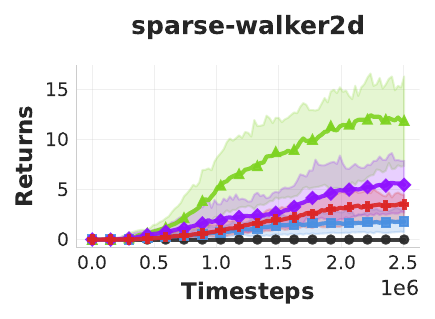} &
        \includegraphics[width=0.23\textwidth]{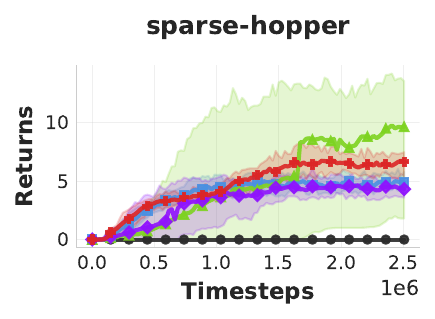} &
        \includegraphics[width=0.23\textwidth]{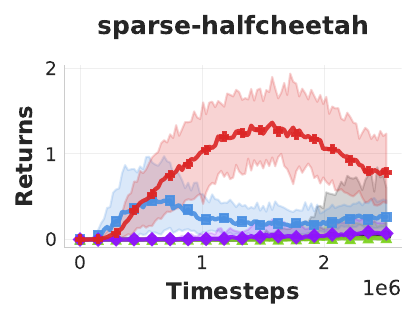} 
        \\[1em]
        
        \multicolumn{4}{c}{\includegraphics[width=0.75\textwidth]{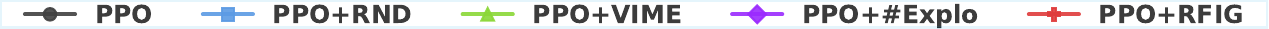}}
    \end{tabular}

    \caption{Comparing exploration strategies. Solid lines represent the interquartile mean, with shaded areas indicating the 25th-75th percentiles across 32 random seeds.}
    \label{fig:rl_results}
\end{figure}

\paragraph{Discussion.} The experimental results across tasks of varying dimensionality reveal distinct performance patterns that highlight fundamental tradeoffs in exploration strategy design (Table \ref{tab:method_comparison}). \texttt{PPO+RFIG} demonstrates remarkable consistency, outperforming vanilla \texttt{PPO}, which demonstrates that our exploration bonus is effective. The key question is how well it performs relative to established methods. \texttt{RFIG} is competitive with NN-based methods across all environments. From control tasks to complex navigation challenges, \texttt{RFIG} maintains stable performance without catastrophic failures.
The representation-dependent methods exhibit environment-specific behavior. \texttt{VIME} excels in some high-dimensional continuous control tasks like sparse-ant and sparse-walker2d when its Bayesian NN representations seem to succeed, but struggles in some other tasks, which can be a failure in capturing a good representation. \texttt{\#Explo} shows similar inconsistency. We believe that both approaches could likely improve with environment-specific hyperparameter tuning, which undermines their general direct applicability.
In contrast, \texttt{RFIG} avoids this representational brittleness. The comparison with \texttt{RND} proves particularly revealing since both employ random features through different mechanisms. Beyond empirical performance, \texttt{RFIG} offers practical advantages through its closed-form solution approach. Our JAX implementation\footnote{The experiment code is available at \url{https://github.com/riiswa/rfig}.} shows no significant execution time differences compared to NN approaches, even with full matrix inversion used in our final experiments. \texttt{RFIG} competitively matches or exceeds \texttt{RND}'s performance while offering superior theoretical interpretability. Moreover, \texttt{RFIG} achieves peak performance in challenging environments like sparse-halfcheetah and Breakout, demonstrating effectiveness across diverse reward structures.

\begin{table}[h]
\centering
\begin{tabular}{lcccc}
\toprule
\textbf{Method} & \textbf{Principle} & \textbf{Representation} & \textbf{Computation} & \textbf{Sensitivity} \\
\midrule
RND & Prediction error & \xmark & Distillation & Medium\\
VIME & Information gain & \cmark & Bayesian NN & High \\
\#Explo & Count-based & \cmark & Auto-encoder & High \\
\textbf{RFIG} & \textbf{Information gain} & \xmark & \textbf{Closed-form} & \textbf{Low} \\
 & \textbf{(Kernel theory)} & & & \\
\bottomrule
\end{tabular}
\caption{Comparison of exploration methods across key characteristics.}
\label{tab:method_comparison}
\end{table}

\section{Conclusion} 

We introduced the Random Feature Information Gain (\texttt{RFIG}) exploration bonus, demonstrating that principled kernel methods can match the empirical performance of neural network-based exploration while offering some theoretical insights and closed-form solutions. By leveraging Bayesian kernel methods and random Fourier features, \texttt{RFIG} achieves competitive results across diverse domains without requiring representation learning or hyperparameter tuning.
Our work challenges the prevailing assumption that effective exploration requires increasingly sophisticated neural architectures, which often have a black-box aspect. \texttt{RFIG}'s success stems from its theoretical foundation in kernel methods, providing mathematical rigor often absent in popular deep RL exploration strategies

\paragraph{Broader impact.} The framework underlying \texttt{RFIG} and its theoretical guarantees extends naturally to other domains requiring uncertainty quantification. Active learning, Bayesian optimization, out-of-distribution detection in offline RL, and other sequential decision-making problems can all benefit from similar kernel-based approaches.

\paragraph{Future directions.} Several promising research avenues emerge from this work. First, studying how \texttt{RFIG} behaves in very high-dimensional spaces like big images will reveal the scalability limits of our kernel-based approach and identify potential adaptations needed for visual domains. Second, developing adaptive mechanisms for kernel parameter selection, particularly length-scale tuning, to potentially improve performance. Third, while we focused on RBF kernels and their RFFs approximation, exploring alternative kernels through different random projection schemes offers exciting possibilities. Specialized kernels might close the performance gap with deep RL methods that currently outperform RFIG in certain environments. Finally, the same information-theoretic principles could enable conservative exploration strategies for offline RL, where avoiding out-of-distribution states takes precedence over optimistic exploration.

\acks{We thank Emilie Kaufmann for valuable feedback on this work. The authors are affiliated with the Inria Scool team project. This work has been supported by the French Ministry of Higher Education and Research, the Hauts-de-France region, Inria, and the MEL. Additional support was provided by the French National Research Agency under the PEPR IA FOUNDRY project (ANR-23-PEIA-0003) and the ANR JCJC  REPUBLIC project (ANR-22-CE23-0003-01).
Experiments presented in this paper were carried out using the PlaFRIM experimental testbed, supported by Inria, CNRS (LABRI and IMB), Université de Bordeaux, Bordeaux INP and Conseil Régional d'Aquitaine (see \url{https://www.plafrim.fr}).}

\bibliography{bibliography}

\newpage

\appendix

\begin{center}
\LARGE\bfseries Supplementary Materials
\end{center}

\section{Full Proofs} \label{appendix:proofs}

\subsection{Information Gain via Random Features}

\begin{proof}[Poof of Proposition~\ref{prop:rf-information-gain}]
 Using the random feature approximation $\mathbf{K}_n \approx \boldsymbol{\Phi}_n \boldsymbol{\Phi}_n^T$, the posterior variance becomes
$$\sigma_n^2(x_*) = \boldsymbol{\phi}(x_*)^T \boldsymbol{\phi}(x_*) - \boldsymbol{\phi}(x_*)^T (\boldsymbol{\Phi}_n \boldsymbol{\Phi}_n^T + \sigma^2 \mathbf{I}_n)^{-1} \boldsymbol{\phi}(x_*)$$
Applying the Woodbury identity:
$$(\boldsymbol{\Phi}_n \boldsymbol{\Phi}_n^T + \sigma^2 \mathbf{I}_n)^{-1} = \frac{1}{\sigma^2}\mathbf{I}_n - \frac{1}{(\sigma^2)^2}\boldsymbol{\Phi}_n(\boldsymbol{\Phi}_n^T \boldsymbol{\Phi}_n + \sigma^2\mathbf{I}_D)^{-1}\boldsymbol{\Phi}_n^T$$
Substituting and simplifying:
\begin{align}
\sigma_n^2(x_*) &= \boldsymbol{\phi}(x_*)^T \boldsymbol{\phi}(x_*) - \frac{1}{\sigma^2}\boldsymbol{\phi}(x_*)^T \boldsymbol{\phi}(x_*) + \frac{1}{\sigma^2}\boldsymbol{\phi}(x_*)^T (\boldsymbol{\Phi}_n^T \boldsymbol{\Phi}_n + \sigma^2\mathbf{I}_D)^{-1} \boldsymbol{\phi}(x_*) \\
&= \sigma^2 \boldsymbol{\phi}(x_*)^T (\boldsymbol{\Phi}_n^T \boldsymbol{\Phi}_n + \sigma^2\mathbf{I}_D)^{-1} \boldsymbol{\phi}(x_*) \\
\frac{\sigma_n^2(x_*)}{\sigma^2} &= \boldsymbol{\phi}(x_*)^T (\boldsymbol{\Phi}_n^T \boldsymbol{\Phi}_n + \sigma^2\mathbf{I}_D)^{-1} \boldsymbol{\phi}(x_*)
\end{align}
Substituting back into the IG formula of GP yields
$\dfrac{1}{2} \log\left(1 + \boldsymbol{\phi}(x_*)^T (\boldsymbol{\Phi}_n^T \boldsymbol{\Phi}_n + \sigma^2\mathbf{I}_D)^{-1} \boldsymbol{\phi}(x_*)\right)$.
We can finally reinterpret the observation noise variance $\sigma^2$ as a regularization parameter $\lambda$, giving the desired result.
\end{proof}

\subsection{Posterior Variance Error Bound}

\begin{proof}[Proof of Proposition \ref{prop:variance-error}]
Let consider the true posterior variance, $\sigma^2_n(x) =  k(x, x) - \mathbf{k}^\top\mathbf{K}^{-1}\mathbf{k}$, with $\mathbf{k} = \mathbf{k}_n(x)$ and $\mathbf{K} = \mathbf{K}_n + \lambda I_n$, considering the approximation $k(x, x') \approx \boldsymbol{\phi}(x)^T \boldsymbol{\phi}(x')$, we can consider our approximated posterior variance $\hat{\sigma}^2_n(x)$ as a perturbation of the true one and define  $\hat{\mathbf{k}} = \mathbf{k} + \boldsymbol{\Delta}_\mathbf{k}$ and $\hat{\mathbf{K}} = \mathbf{K} + \boldsymbol{\Delta}_\mathbf{K}$. Let us expand the following difference 
\begin{align}
\hat{\mathbf{k}}^\top\hat{\mathbf{K}}^{-1}\hat{\mathbf{k}} - \mathbf{k}^\top\mathbf{K}^{-1}\mathbf{k} &= (\mathbf{k} + \boldsymbol{\Delta}_\mathbf{k})^\top(\mathbf{K} + \boldsymbol{\Delta}_\mathbf{K})^{-1}(\mathbf{k} + \boldsymbol{\Delta}_\mathbf{k}) - \mathbf{k}^\top\mathbf{K}^{-1}\mathbf{k} \\
&= \mathbf{k}^\top\hat{\mathbf{K}}^{-1}\mathbf{k} + \mathbf{k}^\top\hat{\mathbf{K}}^{-1}\boldsymbol{\Delta}_\mathbf{k} + \boldsymbol{\Delta}_\mathbf{k}^\top\hat{\mathbf{K}}^{-1}\mathbf{k} + \boldsymbol{\Delta}_\mathbf{k}^\top\hat{\mathbf{K}}^{-1}\boldsymbol{\Delta}_\mathbf{k} - \mathbf{k}^\top\mathbf{K}^{-1}\mathbf{k} \\
&= \mathbf{k}^\top(\hat{\mathbf{K}}^{-1} - \mathbf{K}^{-1})\mathbf{k} + 2\mathbf{k}^\top\hat{\mathbf{K}}^{-1}\boldsymbol{\Delta}_\mathbf{k} + \boldsymbol{\Delta}_\mathbf{k}^\top\hat{\mathbf{K}}^{-1}\boldsymbol{\Delta}_\mathbf{k} \\
&\leq |
\underbrace{
\mathbf{k}^\top(\hat{\mathbf{K}}^{-1} - \mathbf{K}^{-1})\mathbf{k}
}_{\text{Matrix perturbation } t_1}
| + |
\underbrace{
2\mathbf{k}^\top\hat{\mathbf{K}}^{-1}\boldsymbol{\Delta}_\mathbf{k}
}_{\text{Cross term } t_2}
|
+ |
\underbrace{
\boldsymbol{\Delta}_\mathbf{k}^\top\hat{\mathbf{K}}^{-1}\boldsymbol{\Delta}_\mathbf{k}
}_{\text{Vector term } t_3}
|.
\end{align}
where we used the symmetry property $\mathbf{k}^\top\hat{\mathbf{K}}^{-1}\boldsymbol{\Delta}_\mathbf{k} = \boldsymbol{\Delta}_\mathbf{k}^\top\hat{\mathbf{K}}^{-1}\mathbf{k}$ and triangle inequality.

\paragraph{Bounding $t_1$: } Since the smallest eigenvalue of  $\mathbf{K}^{-1}$ $\hat{\mathbf{K}}^{-1}$ is $\lambda$, $\Vert \mathbf{k} \Vert_2 \leq \sqrt{n}\kappa$, and using the inverse matrix perurbation bound for $\hat{\mathbf A} = \mathbf A + \mathbf E$,  $\Vert \hat{\mathbf{A}}^{-1} - \mathbf A^{-1} \Vert_2 \leq \Vert \mathbf A^{-1} \Vert_2 \cdot \Vert \hat{\mathbf{A}}^{-1} \Vert_2 \cdot \Vert \mathbf{E} \Vert_2$, we have
\begin{align}
|\mathbf{k}^\top(\hat{\mathbf{K}}^{-1} - \mathbf{K}^{-1})\mathbf{k}| \leq \Vert  \hat{\mathbf{K}}^{-1} \Vert_2 \cdot \Vert  \boldsymbol{\Delta}_\mathbf{K} \Vert_2 \cdot  \Vert  \mathbf{K}^{-1} \Vert_2 \cdot \Vert \mathbf{k} \Vert_2^2 \leq \frac{\epsilon \kappa^2 n^2}{\lambda^2 }.
\end{align}

\paragraph{Bounding $t_2$ and $t_3$: } Similarly to $t_1$, we can bound the two other terms with
\begin{align}
    | 2\mathbf{k}^\top\hat{\mathbf{K}}^{-1}\boldsymbol{\Delta}_\mathbf{k}| \leq \frac{2\epsilon n \kappa}{\lambda}, \qquad  
|\boldsymbol{\Delta}_\mathbf{k}^\top\hat{\mathbf{K}}^{-1}\boldsymbol{\Delta}_\mathbf{k} | \leq \frac{\epsilon^2n}{\lambda}.
\end{align}
Finally, $|\hat{\sigma}^2_n(x) - \sigma^2_n(x)| \leq \epsilon + \dfrac{\epsilon \kappa^2 n^2}{\lambda^2 } + \dfrac{2\epsilon n \kappa}{\lambda} + \dfrac{\epsilon^2n}{\lambda} = \epsilon + \dfrac{\epsilon  \kappa^2}{ \lambda_0^2} + \dfrac{2\epsilon\kappa}{\lambda_0} + \dfrac{\epsilon^2}{\lambda_0}$, using Assumption~\ref{ass:regularization}.
\end{proof}

\subsection{RFIG Error Bound}

To bound RFIG, we can directly use the established bound for posterior variance, by using the following lemma:
\begin{lemma}[Shifted Logarithmic Difference bound]\label{lem:geometric_mean_bound}
For any $a, b > 0$, we have
\begin{equation}
|\log(1+a) - \log(1+b)| \leq |a-b|
\end{equation}
\end{lemma}
\begin{proof}
On the interval between $a$ and $b$, there exists $c$ between $a$ and $b$ such that
$
\log(1+a) - \log(1+b) = f'(c)(a-b) = \frac{a-b}{1+c}.
$
Since $\min(a,b) \leq c \leq \max(a,b)$, we have
$\frac{1}{1+\max(a,b)} \leq \frac{1}{1+c} \leq \frac{1}{1+\min(a,b)}.
$
Therefore,
$
\left|\frac{a-b}{1+c}\right| \leq \frac{|a-b|}{1+\min(a,b)},
$ we have $a, b >0$,
which completes the proof.
\end{proof}
Now, we have all the elements to obtain a deterministic upper bound on RFIG.
\begin{proof}[Proof of Proposition \ref{prop:ig-error}]
By applying Lemma X, we have
\begin{align}
    |\hat{\operatorname{IG}}(x|\mathcal{D}_n) - \operatorname{IG}(x|\mathcal{D}_n) | \leq \frac{\boldsymbol{\Delta}_{\sigma^2_n}}{2\lambda} 
    & = \frac{\epsilon + \frac{\epsilon \kappa^2}{\lambda_0^2} + \frac{2\epsilon\kappa}{\lambda_0} + \frac{\epsilon^2}{\lambda_0}}  {2\lambda} \\
    &= \frac{\epsilon\left[\left(1 + \frac{\kappa}{\lambda_0}\right)^2 + \frac{\epsilon}{\lambda_0}\right]}{2\lambda} \\
    &= \frac{\epsilon\left[\frac{(\lambda_0 + \kappa)^2 + \epsilon\lambda_0}{\lambda_0^2}\right]}{2\lambda} \\
    &= \frac{\epsilon(\lambda_0+\kappa)^2+\epsilon^2\lambda_0}{2n\lambda_0^3}
\end{align}
That ends the proof.
\end{proof}

\subsection{High Probability RFIG Bound}

\begin{proof}[Proof of Proposition \ref{cor:feature-requirement}]
    We aim to find when the information gain error is at least $\varepsilon$:
    \begin{align}
        \frac{\epsilon(\lambda_0+\kappa)^2+\epsilon^2\lambda_0}{2n\lambda_0^3} &\geq \varepsilon \\
        \epsilon(\lambda_0+\kappa)^2 + \epsilon^2\lambda_0 &\geq \varepsilon 2n\lambda_0^3 \\
        \epsilon(\lambda_0+\kappa)^2 + \epsilon^2\lambda_0 - \varepsilon 2n\lambda_0^3   &\geq 0
    \end{align}
    This is a quadratic inequality in $\epsilon$. The quadratic $f(\epsilon) = \lambda_0\epsilon^2 + (\lambda_0+\kappa)^2\epsilon - 2n\lambda_0^3\varepsilon$ has for root:
    \begin{equation}
        \epsilon \geq \frac{- (\lambda_0 + \kappa)^2 + \sqrt{(\lambda_0 + \kappa)^4 + 8n\lambda_0^4\varepsilon}}{2\lambda_0},
    \end{equation}
    since $\lambda_0 > 0$, the parabola opens upward. Setting $\kappa=1$ (Proposition \ref{prop:rff-convergence}), ends the proof.
\end{proof}

\section{Numerical Experiments}

\subsection{Newton-Schulz iterations} \label{appendix:newton}

The Newton-Schulz method provides an iterative approach to matrix inversion that is particularly well-suited for our kernel matrix updates.
\begin{algorithm2e}[h]
\caption{Newton-Schulz Matrix Inversion Update}
\KwIn{Previous inverse $\mathbf{X}_{old}$, matrix update $\boldsymbol{\Phi}_t$, regularization $\lambda$}
$\mathbf{A} \leftarrow \boldsymbol{\Phi}_t^T\boldsymbol{\Phi}_t + \lambda\mathbf{I}$\;
$\mathbf{X}_0 \leftarrow \mathbf{X}_{old}$ (warm start)\;
\For{$k = 1, 2, \ldots, K$}{
    $\mathbf{X}_k \leftarrow \mathbf{X}_{k-1}(2\mathbf{I} - \mathbf{A}\mathbf{X}_{k-1})$\;
}
\Return $\mathbf{X}_K$
\end{algorithm2e}
Due to JAX's compilation and parallelization constraints, we implement a fixed number of Newton-Schulz iterations ($K = 20$) rather than iterating until convergence. In practice, we observe that 20 iterations provide sufficient accuracy for information gain estimation while maintaining computational efficiency across all experimental environments.

\subsection{Hyperparameter Configuration} \label{appendix:ppo_hyperparams}

Table~\ref{tab:ppo_hyperparams} presents the complete hyperparameter configuration used for PPO experiments across all environments. For RND baseline comparisons, we use an embedding size of 256, hidden layer sizes of (256, 256), a bonus learning rate of 1e-4, with ReLU activations, following standard RND implementation practices.

\begin{table}[h]
\centering
\caption{PPO hyperparameters used in all experiments.}
\label{tab:ppo_hyperparams}
\begin{tabular}{lr}
\toprule
\textbf{Parameter} & \textbf{Value} \\
\midrule
\multicolumn{2}{l}{\textit{Training Configuration}} \\
Total timesteps & 1,000,000 \\
Number of environments & 32 \\
Steps per environment & 128 \\
Evaluation frequency & 24,576 \\
Anneal learning rate & True \\
\midrule
\multicolumn{2}{l}{\textit{PPO Algorithm}} \\
Learning rate & 0.0003 \\
Number of epochs & 4 \\
Number of minibatches & 32 \\
Clip ratio ($\epsilon$) & 0.2 \\
Value function coefficient & 0.5 \\
Entropy coefficient & 0.01 \\
Maximum gradient norm & 0.5 \\
\midrule
\multicolumn{2}{l}{\textit{GAE \& Discounting}} \\
Discount factor ($\gamma$) & 0.99 \\
GAE lambda ($\lambda$) & 0.95 \\
\midrule
\multicolumn{2}{l}{\textit{Normalization}} \\
Normalize observations & True \\
Normalize intrinsic rewards & True \\
\midrule
\multicolumn{2}{l}{\textit{Network Architecture}} \\
Activation function & Tanh \\
Hidden layer sizes & (64, 64) \\
\bottomrule
\end{tabular}
\end{table}

\subsection{Milestone Reward Wrapper} \label{appendix:milestone_wrapper}

We implement a \texttt{MilestoneRewardWrapper} that transforms dense reward signals into sparse, milestone-based rewards. This wrapper provides rewards only when the agent reaches specific distance milestones during locomotion, creating challenging exploration scenarios where traditional dense rewards are unavailable. The wrapper operates by tracking the agent's forward displacement from its initial position and providing rewards at fixed distance intervals. Specifically, it:
\begin{enumerate}
    \item Records the agent's initial position at environment reset
    \item Monitors the agent's current position throughout the episode
    \item Calculates the total distance traveled as the difference between current and initial positions
    \item Awards rewards when the agent crosses predefined distance milestones
\end{enumerate}
The milestone reward $r_t$ at timestep $t$ is computed as:
\begin{equation}
r_t = \begin{cases}
\alpha \cdot (m_t - m_{t-1}) & \text{if } m_t > m_{t-1} \\
0 & \text{otherwise}
\end{cases}
\end{equation}
where $m_t = \lfloor d_t / \delta \rfloor$ represents the current milestone, $d_t$ is the distance traveled, $\delta$ is the milestone distance interval, and $\alpha$ is the reward scale factor.
The wrapper accepts three key parameters:
\begin{itemize}
    \item \textbf{milestone\_distance} ($\delta = 1.0$): Distance interval between consecutive milestones
    \item \textbf{reward\_scale} ($\alpha = 1.0$): Scale factor applied to milestone rewards
    \item \textbf{position\_fn}: Function extracting agent position from environment state (defaults to x-coordinate of the first body)
\end{itemize}
This design creates environments where agents receive no immediate feedback for small movements but are rewarded for achieving meaningful locomotion progress, making these tasks particularly challenging for exploration strategies. For reproducibility, we provide the complete implementation of the \texttt{MilestoneRewardWrapper}:

\begin{lstlisting}[language=Python, caption=MilestoneRewardWrapper Implementation, label=lst:milestone_wrapper, basicstyle=\small\ttfamily, breaklines=true, frame=single]
from typing import Callable, Optional
from brax.envs import PipelineEnv, State, Wrapper
import jax
from jax import numpy as jp

class MilestoneRewardWrapper(Wrapper):
    """Wrapper that adds milestone-based rewards to any Brax environment.
    This wrapper gives a reward whenever the agent reaches specified distance
    milestones (e.g., every 1.0 unit of forward movement).
    """
    def __init__(
            self,
            env: PipelineEnv,
            milestone_distance: float = 1.0,
            reward_scale: float = 1.0,
            position_fn: Optional[Callable[[State], jp.ndarray]] = lambda state: state.pipeline_state.x.pos[0, 0],
    ):
        """Initializes the milestone reward wrapper.
        Args:
          env: The environment to wrap.
          milestone_distance: Distance between reward milestones.
          reward_scale: Scale factor for milestone rewards.
          position_fn: Function that extracts position from state.
                       Default extracts x position from first body.
        """
        super().__init__(env)
        self._milestone_distance = milestone_distance
        self._reward_scale = reward_scale
        self._position_fn = position_fn

    def reset(self, rng: jax.Array) -> State:
        """Resets the environment and initializes milestone reward tracking."""
        state = self.env.reset(rng)
        # Get initial position
        initial_position = self._position_fn(state)
        # Add milestone reward tracking info
        info = state.info.copy()
        info.update({
            'initial_position': initial_position,
            'last_milestone': 0.0,
            'total_milestones': 0,
            'distance_traveled': 0.0,
            'current_milestone': 0.0,
        })
        return state.replace(info=info)

    def step(self, state: State, action: jax.Array) -> State:
        """Steps the environment and adds milestone rewards."""
        # Get tracking info
        initial_position = state.info.get('initial_position')
        last_milestone = state.info.get('last_milestone', 0.0)
        total_milestones = state.info.get('total_milestones', 0)
        
        # Step the environment
        next_state = self.env.step(state, action)
        
        # Get current position and calculate distance traveled
        current_position = self._position_fn(next_state)
        distance_traveled = current_position - initial_position
        
        # Calculate the current milestone
        current_milestone = jp.floor(distance_traveled / self._milestone_distance)
        
        # Check if we've reached a new milestone
        new_milestone_reached = current_milestone > last_milestone
        
        # Calculate milestone reward
        reward = jp.where(
            new_milestone_reached,
            self._reward_scale * (current_milestone - last_milestone),
            0.0
        )
        
        # Update the total milestones count
        total_milestones = jp.where(
            new_milestone_reached,
            total_milestones + jp.int32(current_milestone - last_milestone),
            total_milestones
        )
        
        # Update the last milestone
        last_milestone = jp.where(new_milestone_reached, current_milestone, last_milestone)
        
        # Update info
        info = next_state.info.copy()
        info.update({
            'initial_position': initial_position,
            'last_milestone': last_milestone,
            'total_milestones': total_milestones,
            'distance_traveled': distance_traveled,
            'current_milestone': current_milestone,
        })
        
        return next_state.replace(reward=reward, info=info)
\end{lstlisting}

\end{document}